\PassOptionsToPackage{dvipsnames}{xcolor}
\documentclass[letterpaper, 10 pt, conference]{ieeeconf}

\IEEEoverridecommandlockouts                            
\usepackage{graphics}
\usepackage{epsfig}
\usepackage{amsmath}
\usepackage{amssymb}
\usepackage{color}
\usepackage{bm}
\usepackage{caption}
\usepackage{subfig}
\usepackage{cite}
\usepackage{tikz}

\usepackage{amsthm}

\DeclareMathAlphabet\mathbfcal{OMS}{cmsy}{b}{n}
\newtheorem{theorem}{Theorem}
\newtheorem{assumption}{Assumption}

\makeatletter
\let\NAT@parse\undefined
\makeatother

\usepackage{xcolor}
\usepackage[pagebackref,breaklinks,colorlinks]{hyperref}
\hypersetup{
    colorlinks=true,
    allcolors=ForestGreen,
    pdftitle={An  observer  cascade  for  velocity  and  multiple  line  estimation},
    pdfpagemode=FullScreen,
    }

\usepackage[capitalize,nameinlink]{cleveref}
\crefname{equation}{}{}
\Crefname{equation}{Equation}{Equations}
\crefname{figure}{Fig.}{Figs.}
\Crefname{figure}{Figure}{Figures}
\crefname{table}{Tab.}{Tabs.}
\Crefname{table}{Table}{Tables}
\crefname{section}{Sec.}{Secs.}
\Crefname{section}{Section}{Sections}
\crefname{problem}{Problem}{Problems}
\Crefname{problem}{Problem}{Problems}
\crefname{definition}{Definition}{Definitions}
\Crefname{definition}{Definition}{Definitions}
\crefname{lemma}{Lemma}{Lemmas}
\Crefname{lemma}{Lemma}{Lemmas}
\crefname{theorem}{Thm.}{Thms.}
\Crefname{theorem}{Theorem}{Theorems}
\crefname{remark}{Rmk.}{Rmks.}
\Crefname{remark}{Remark}{Remarks}
\crefname{enumi}{item}{items}
\Crefname{enumi}{Item}{Items}
\crefname{algocf}{Alg.}{Algs.}
\Crefname{algocf}{Algorithm}{Algorithms}
\crefname{assumption}{Asm.}{Asms.}
\Crefname{assumption}{Assumption}{Assumptions}

\usepackage{booktabs}
\usepackage{makecell}

\makeatletter
\def\footnoterule{\relax%
  \kern-5pt
  \hbox to \columnwidth{\vrule width 0.5\columnwidth height 0.4pt\hfill}
  \kern4.6pt}
\makeatother

\title{\LARGE \bf
An observer cascade for velocity and multiple line estimation*
}

\author{Andr\'e Mateus, Pedro U. Lima, and Pedro Miraldo
\thanks{*This work was supported by FCT grant {\tt PD/BD/135015/2017}, LARSyS - FCT Project {\tt UIDB/50009/2020}, and by the European Union from the European Regional Development Fund under the Smart Growth Operational Programme as part of the project {\tt POIR.01.01.01-00-0102/20}, with title ”Development of an innovative, autonomous mobile robot for retail stores”.}
\thanks{A. Mateus, P. U. Lima, and P. Miraldo  are with Institute for Systems and Robotics (ISR/IST), LARSyS, Instituto Superior T\'ecnico, Univ. Lisboa, Portugal \newline
        {\tt\footnotesize \{andre.mateus, pedro.lima, pedro.miraldo\} \newline @tecnico.ulisboa.pt}.}
}

\usepackage{tikz}
\usepackage{lipsum}

\newcommand\copyrighttext{%
  \footnotesize © 2022 IEEE.  Personal use of this material is permitted.  Permission from IEEE must be obtained for all other uses, in any current or future media, including reprinting/republishing this material for advertising or promotional purposes, creating new collective works, for resale or redistribution to servers or lists, or reuse of any copyrighted component of this work in other works.}
\newcommand\copyrightnotice{%
\begin{tikzpicture}[remember picture,overlay]
\node[anchor=south,yshift=10pt] at (current page.south) {\fbox{\parbox{\dimexpr\textwidth-\fboxsep-\fboxrule\relax}{\copyrighttext}}};
\end{tikzpicture}%
}

\begin{document}

\maketitle
\copyrightnotice

\begin{abstract}

Previous incremental estimation methods consider estimating a single line, requiring as many observers as the number of lines to be mapped.
This leads to the need for having at least $4N$ state variables, with $N$ being the number of lines.
This paper presents the first approach for multi-line incremental estimation.
Since lines are common in structured environments, we aim to exploit that structure to reduce the state space.
The modeling of structured environments proposed in this paper reduces the state space to $3N + 3$ and is also less susceptible to singular configurations.
An assumption the previous methods make is that the camera velocity is available at all times. 
However, the velocity is usually retrieved from odometry, which is noisy.
With this in mind, we propose coupling the camera with an Inertial Measurement Unit (IMU) and an observer cascade. A first observer retrieves the scale of the linear velocity and a second observer for the lines mapping. 
The stability of the entire system is analyzed.
The cascade is shown to be asymptotically stable and shown to converge in experiments with simulated data.

\end{abstract}
\section{Introduction}
\label{sec:intro}

Depth estimation is a well-studied problem in Robotics and Computer Vision.
When several images are available with a significant baseline between them, the typical solution is to solve the Structure-from-Motion problem \cite{koenderink1991,bartoli2005}.
However, depth can only be recovered up to a common scale factor when using monocular vision.
To achieve metric reconstruction, they require either stereo vision or the 3D structure of a set of features to be known.
To use monocular vision and handle small baselines, some authors took advantage of robotic systems. 
Since an estimate of the velocity is usually available (though noisy), the dynamical model of the motion of the observed features in the image concerning the camera velocity can be exploited.
These approaches are based on filtering or observer design.
Filtering approaches tend to exploit the use of Kalman Filter \cite{kalman1960} and its variants.
Examples of filtering approaches are \cite{civera2008,civera2010,omari2013,smith2006,zhang2011}.

A limitation of the previous filtering approaches is their reliance on EKF, which requires linearization of the dynamics.
Thus, the estimation is susceptible to poor initial estimates, which may cause the filter to diverge.
To address the linearization issues, nonlinear observers have been used.
Observer-based methods consist of exploiting the visual dynamical systems developed in Visual Servoing \cite{chaumette2006} and designing a nonlinear estimation scheme (to retrieve 3D information), given the camera velocities and image measurements.
Observer-based approaches were presented in \cite{dixon2003,deluca2008,morbidi2010,sassano2010,dani2012,rodrigues2019,mateus2021}.

Since the focus of the observer-based approaches has been on robotic applications, an assumption that can be made is that we can control the camera's path.
With that in mind, some authors studied how we can control the camera to improve the estimation, i.e., speed-up convergence and prevent ill-conditioned poses.
An \emph{Active Structure-from-Motion} was presented in \cite{spica2013}.
This framework has been successfully applied to different visual features.
Namely, points, cylinders and spheres in \cite{spica2014}, planes and image moments in \cite{spica2015,spica2015b}, rotation invariants in \cite{tahri2015,tahri2017}, and lines in \cite{mateus2018,mateus2019,mateus2021}.
A method to couple this framework with a Visual Servoing control loop is presented in \cite{spica2017}.
An alternative framework for active depth estimation of points is presented in \cite{rodrigues2020}, where convergence guarantees are provided, with fewer constraints on the camera motion relative to \cite{spica2013}.

The previous approaches to depth estimation assume that the camera velocities (system inputs) are known at every time.
However, velocity estimated by the robot's odometry or joint encoders tends to yield noisy readings.
A solution is to attempt to estimate the camera velocity at the same time as the depth.
A method that allows recovering both linear velocity and point depth exploiting an unknown input observer is presented in \cite{benyoucef2019,benyoucef2021}.
However, they required the angular acceleration to be known.

The majority of the state-of-the-art depth observers consider a single feature and are based on the assumption that the velocity of the camera (both linear and angular) is available, \cite{spica2013,rodrigues2020,mateus2021}.
This work is the first to address the multiple-line estimation problem using a nonlinear observer. 
An issue arising in multiple-feature scenarios is ill-posed configurations, which can make the observer diverge and become unable to recover.
To prevent such configurations, one can make assumptions about the environment.
Since we are concerned with reconstructing 3D lines, a common feature in human-made environments, the Manhattan World assumption should hold.
This consists of assuming that the lines and planes in the scene are either parallel or perpendicular to each other.
Thus, in this work, we exploit the Manhattan World assumption for multiple-line estimation.

Even though most robots provide velocity readings, e.g., from encoders placed in the joint motors, those are usually noisy.
Since the observer takes the velocity readings at face value, it is susceptible to that noise.
Given the low price of Inertial Measurement Units (IMU), those have become widely available and present in robots. However, noise is present in IMU, as, in any sensor, these are much more accurate than traditional motor encoders in estimating angular velocity.
Furthermore, they provide linear acceleration readings, which can be integrated or filtered to estimate the linear velocity.
With that in mind, this work aims to exploit visual-inertial systems to estimate both depth and linear velocity by employing an observer cascade.
\section{Background}
\label{sec:probDefandBack}

For completeness, we present the background material used in this paper. We start by describing the equations of motion of a line feature as presented in \cite{andreff2002}. 
Then, the first observer in the cascade, which takes optical flow measurement, and the IMU readings to estimate the camera linear velocity, is presented.

\subsection{Equations of motion of 3D Lines}
\label{sec:lineDynamics}

Let us consider a line defined with binormalised Pl\"ucker coordinates, presented in \cite{andreff2002}.
In those coordinates, a line is defined by a unit direction vector $\mathbf{d}$, a unit moment vector\footnote{The line moment is a vector perpendicular to the plane defined by the line and the camera optical center.} $\mathbf{n}$, and the line depth\footnote{The line depth is defined as the shortest distance from the camera to the line.} $l$.
The dynamics of the binormalised Pl\"ucker coordinates are given as
\begin{align}
    \dot{\mathbf{d}} & =  \bm{\omega}_c \times \mathbf{d} \label{eq:ddyn} \\
    \dot{\mathbf{n}} & = \bm{\omega}_c \times \mathbf{n} - \frac{\bm{\nu}_c^T\mathbf{n}}{l}( \mathbf{d}\times \mathbf{n}) \label{eq:hdyn} \\
    \dot{l} & =  \bm{\nu}_c^T(\mathbf{d} \times \mathbf{n}), \label{eq:depthdyn}\
\end{align}
where $\bm{\nu}_c\in\mathbb{R}^3$ is the camera linear velocity, and $\bm{\omega}_c\in\mathbb{R}^3$ is the camera angular velocity. Thus, the dynamical systems is described by \cref{eq:ddyn}, \cref{eq:hdyn}, and \cref{eq:depthdyn}, the inputs are the camera velocities $\mathbf{v}_c =  \bm{\nu}_c,\ \bm{\omega}_c$, and the output is the normalized moment vector $\mathbf{y} = \mathbf{n}$.
The state parameters to be estimated are $l$ and $\mathbf{d}$. 

\subsection{An Observer for Plane Depth and Linear Velocity Estimation}
\label{sec:velPlaneDepthEst}

Let $\mathcal{C}$ denote the camera frame, $\mathcal{W}$ denote the world frame, $\mathcal{I}$ represent the IMU frame, $\mathbf{R}_{\mathcal{I}}^\mathcal{C}$ is the rotation from the IMU to camera frame, and $\mathbf{t}_{\mathcal{I}}^\mathcal{C}$ is the translation between both frames.
Furthermore, let the camera be looking at a planar surface. 
By combining the IMU measures with Optical flow, it is possible to obtain a simplified continuous homography matrix \cite{ma2012}.
Let us define a plane $\bm{\pi} = [\mathbf{m}, \rho]$, where $\mathbf{m}$ is the unit normal vector to the plane, and $\rho$ the plane offset.
Then, decomposing this matrix, an estimate for the plane normal $\mathbf{m}$, and the camera linear velocity scaled by the inverse plane depth $\frac{\bm{\nu}_c}{\rho}$ can be obtained as shown in \cite{grabe2015}.
Thus, the goal is to retrieve the plane depth, given the angular velocity, linear acceleration, the plane normal, and the linear velocity scaled by the inverse plane depth.

Since we are interested in exploiting nonlinear state observers, let us define the state vector.
The state is given as $[\mathbf{s},\psi]$, with $\mathbf{s} = \frac{\bm{\nu}_c}{\rho}$, and $\psi = \frac{1}{\rho}$.
The dynamics is presented in \cite{grabe2015}, and is given as
\begin{equation}
\begin{split}
    \dot{\mathbf{s}} & = -[\bm{\omega}_c]_{\times}\mathbf{s} + -\mathbf{s}\mathbf{s}^T\mathbf{m} + (\mathbf{R}_{\mathcal{I}}^\mathcal{C} \mathbf{a}_{\mathcal{I}} + [\bm{\omega}_c]_{\times}^2 \mathbf{t}_{\mathcal{I}}^\mathcal{C}) \psi \\
    \dot{\psi} & = -\psi \mathbf{s}^T\mathbf{m},
\end{split}
\label{eq:velPlaneDepthDyn}
\end{equation}
where $\mathbf{a}_{\mathcal{I}}$ is the linear acceleration in IMU frame, after compensating for the gravity acceleration.
One way to compensate for the gravity acceleration is to use magnetometer measurements to retrieve the gravity direction. Then, that direction is aligned with one of the axis (usually the $z$ axis) by computing $\mathbf{R}_\mathcal{W}^\mathcal{I}$. Finally, we obtain $\mathbf{a}_\mathcal{I} = \mathbf{f}_\mathcal{I} + \mathbf{R}_\mathcal{W}^\mathcal{I} [0,0,g]^T$.
An observer for the system in \cref{eq:velPlaneDepthDyn} is given as
\begin{equation}
    \begin{split}
        \dot{\hat{\mathbf{s}}} & = -[\bm{\omega}_c]_{\times}\mathbf{s} + -\mathbf{s}\mathbf{s}^T\mathbf{m} + \bm{\Omega}^T \hat{\psi} + k_{\mathbf{s}}\tilde{\mathbf{s}} \\
        \dot{\hat{\psi}} & = -\hat{\psi} \mathbf{s}^T\mathbf{m} + k_{\rho}\bm{\Omega}\tilde{\mathbf{s}},
    \end{split}
    \label{eq:velPlaneDepthObs}
\end{equation}
where $\tilde{\mathbf{s}} = \mathbf{s} - \hat{\mathbf{s}}$, $k_{\mathbf{s}}$ \& $k_{\rho}$ are positive gains, and $\bm{\Omega} = (\mathbf{R}_{\mathcal{I}}^\mathcal{C} \mathbf{a}_{\mathcal{I}} + [\bm{\omega}_c]_{\times}^2 \mathbf{t}_{\mathcal{I}}^\mathcal{C})^T$.
The observer in \cref{eq:velPlaneDepthObs} is shown to be exponentially stable in \cite{grabe2015}.
\section{Lines in Manhattan World}
\label{sec:mwLines}

This section exploits the Manhattan World assumption \cite{Coughlan1999,Coughlan2000} to derive a new dynamic model of lines in the Manhattan World.
Let us now consider that we have a set of lines in Manhattan World.
We can compute the orthogonal directions by using, for example, a state-of-the-art method as \cite{lu2017}.
By stacking the direction vectors, we can obtain an orthonormal basis, which is given as
\begin{equation}
    \mathbf{R}_{\mathcal{C}}^{\mathcal{W}} = \begin{bmatrix} \mathbf{d}_1^T \\ \mathbf{d}_2^T \\ \mathbf{d}_3^T \end{bmatrix}.
    \label{eq:dortho}
\end{equation}
Since, $\mathbf{R}_{\mathcal{C}}^{\mathcal{W}}$ is an orthogonal matrix which can be represented with a reduced number of degrees-of-freedom, from nine (three vectors in $\mathbb{R}^3$) to three.
This reduced representation can be achieved by using Cayley parameters.
Starting from the orthogonal matrix $\mathbf{R}_{\mathcal{C}}^{\mathcal{W}}$ the Cayley parameters can be recovered using the Cayley transform \cite{Krantz2012}:
\begin{equation}
    \mathbf{G} = (\mathbf{R}_{\mathcal{C}}^{\mathcal{W}} - \mathbf{I})(\mathbf{R}_{\mathcal{C}}^{\mathcal{W}} + \mathbf{I})^{-1},
    \label{eq:cayleytransform}
\end{equation}
where $\mathbf{I}$ is an identity matrix,
\begin{equation}
    \mathbf{G} = \begin{bmatrix} 0 & -c_3 & c_2 \\ c_3 & 0 & -c_1 \\ -c_2 & c_1 & 0 \end{bmatrix}
    \label{eq:cayleyskew}
\end{equation}
is a skew-symmetric matrix, and $\mathbf{c} = \begin{bmatrix}c_1 & c_2 & c_3\end{bmatrix}^T$ are the Cayley parameters.
These parameters allow us to reduce the size of the state space of the dynamical system in \cref{sec:mwdyn}.
Since we can use three variables (Cayley parameters) instead of nine (entries of the rotation matrix) to represent the camera rotation in the state.

If we now project the moment vector, of each line, onto $\mathbf{R}_{\mathcal{C}}^{\mathcal{W}}$, we can reduce its degrees-of-freedom from three to two.
Thus, let us define these projections as
\begin{equation}
    \mathbf{o}_i = \mathbf{R}_{\mathcal{C}}^{\mathcal{W}}\mathbf{n}_i.
    \label{eq:tau}
\end{equation}
From the orthogonality of the Pl\"ucker coordinates\footnote{Notice that since the moment is perpendicular to a plane containing the line, it is also perpendicular to the line direction.}, we can conclude that the coordinate of the vector $\mathbf{o}_i$ corresponding to the projection of the moment to the respective direction is equal to zero.
Taking the inverse of \cref{eq:tau} we can define the moment vector as
\begin{equation}
    \mathbf{n}_i = o_{i1}\mathbf{d}_1 + o_{i2}\mathbf{d}_2 + o_{i3}\mathbf{d}_3,
    \label{eq:hlincomb}
\end{equation}
with $i= 1,...,N$, and $N$ being the number of lines. Furthermore, $o_{ij}, j = 1,2,3$ are the entries of $\mathbf{o}_i$ corresponding to the projection of $\mathbf{n}_i$ to each of the three principal directions.
Depending on the direction of the line $i$, one of $o_{ij}$ will be equal to zero.

\subsection{Dynamics of the Manhattan World Lines}
\label{sec:mwdyn}

Let us define the unknown variables to be
\begin{equation}
    \chi_i = \frac{1}{l_i},
    \label{eq:mwchi}
\end{equation}
where $l_i$ is the depth of $i^{th}$ line.
Let us assume that the lines have been assigned to one of the orthogonal directions.
Furthermore, when applying \cref{eq:tau}, we know which component of vector $\mathbf{o}_i$ is equal to zero.
To design an observer for the system, whose state is comprised of the parameters $o_{ij}$, $c_j$, and $\chi_i$, with $i = 1,...,N$, and $j = 1,2,3$, we must compute the dynamics of our system.

Applying the quotient derivative rule and replacing \cref{eq:depthdyn} and \cref{eq:mwchi}, we obtain the dynamics of $\chi_i$ as
\begin{equation}
    \dot{\chi_i} = \bm{\nu}_c^T (\mathbf{n}_i \times \mathbf{d}_j)\chi_i^2.
    \label{eq:chii_derivative}
\end{equation}
Then, replacing \cref{eq:hlincomb} in \cref{eq:chii_derivative}, yields
\begin{equation}
    \dot{\chi_i} = \bm{\nu}_c^T ( (o_{i1}\mathbf{d}_1 + o_{i2}\mathbf{d}_2 + o_{i3}\mathbf{d}_3) \times \mathbf{d}_j)\chi_i^2 .
    \label{eq:chii_derivativeFull}
\end{equation}
The line $i$ can belong to one of the three directions. So depending on the value of $j$, one obtains the dynamics of the inverse depth of the line.  

The three directions of the Manhattan frame are the basis vectors of the orthogonal matrix $\mathbf{R}_{\mathcal{C}}^{\mathcal{W}}$, from which the Cayley parameters are retrieved.
To compute the dynamics of these parameters, let us take \cref{eq:ddyn} and \cref{eq:dortho} to obtain the time derivative of $\mathbf{R}_{\mathcal{C}}^{\mathcal{W}}$, which is given as
\begin{equation}
    \dot{\mathbf{R}}_{\mathcal{C}}^{\mathcal{W}} = -\mathbf{R}_{\mathcal{C}}^{\mathcal{W}} [\bm{\omega_c}]_{\text{x}}.
    \label{eq:dorthodyn}
\end{equation}
By inverting the Cayley transform in \cref{eq:cayleytransform}, computing the time derivative, and, finally, equating to \cref{eq:dorthodyn}, we get the time derivative of the Cayley parameters as
\begin{equation}
    \begin{bmatrix} \dot{c}_1\\ \dot{c}_2\\ \dot{c}_3 \end{bmatrix} = \underbrace{ -\frac{1}{2} \begin{bmatrix} 1 + c_1^2 & c_1c_2 - c_3 & c_1c_3 + c_2 \\ c_1c_2 + c_3 & 1+c_2^2 & c_2c_3 - c_1 \\ c_1c_3 - c_2 & c_2c_3 + c_1 & 1+c_3^2 \end{bmatrix}}_{\mathbf{Q}\in\mathbb{R}^{3\times 3}} \bm{\omega}_c.
    \label{eq:cayleydyn}
\end{equation}

The remaining state variables are the $\mathbf{o}_i$.
Let us compute the time derivative of \cref{eq:tau}. Making use of \cref{eq:ddyn}, \cref{eq:hdyn}, and \cref{eq:tau}, we obtain:
\begin{equation}
    \dot{\mathbf{o}}_i = -\mathbf{R}_{\mathcal{C}}^{\mathcal{W}} [\bm{\omega_c}]_{\text{x}}  \mathbf{n}_i + \mathbf{R}_{\mathcal{C}}^{\mathcal{W}} \left( [\bm{\omega_c}]_{\text{x}} \mathbf{n}_i + \bm{\nu}_c^T\mathbf{n}_i( [\mathbf{n}_i]_{\text{x}} \mathbf{d}_j) \right)\chi_i.
    \label{eq:tau_derivative_full}
\end{equation}
Depending on the principal direction to which the line $i$ belongs, a different coordinate of vector $\mathbf{o}_i$ will be equal to zero.

\section{Observer Design for Multiple Lines}
\label{sec:mwLinesObs}

Let us consider $D_1$ lines with direction $\mathbf{d}_1$, $D_2$ lines with direction $\mathbf{d}_2$, and $D_3$ lines with direction $\mathbf{d}_3$, then $N = D_1 + D_2 + D_3$.
Using the representation in \cref{sec:mwLines} we have $2N$ variables to describe the vectors $\mathbf{o}_i$, three Cayley parameters, and $N$ depths.
Recalling that we can retrieve the lines' moments and MW frame rotation with the Manhattan World assumption, we have $2N+3$ measures and $N$ unknowns.

Furthermore, let $\bm{\tau}_i$ be a two-dimensional vector, with each entry corresponding to the non-zero elements of $\mathbf{o}_i$.
Besides, let the dynamics of the entire system be written as
\begin{equation}
    \begin{split}
        \dot{\mathbf{c}} & = \mathbf{Q}\bm{\omega}_c \\
        \dot{\bm{\tau}}_i & = \mathbf{T}_i(\bm{\tau}_i,\mathbf{c})^T\bm{\nu}_c \chi_i \\
        \dot{\chi}_i & = \mathbf{X}_i(\bm{\tau}_i,\mathbf{c})^T\bm{\nu}_c \chi_i^2,
    \end{split}
    \label{eq:mwdyn}
\end{equation}
with $i = 1,...,N$, $\mathbf{T}_i \in \mathbb{R}^{3\times2}$ is obtained from \cref{eq:tau_derivative_full}, depending on the line direction; and $\mathbf{X}_i \in \mathbb{R}^{3\times1}$ is given by \cref{eq:chii_derivativeFull}.
Defining $\hat{\mathbf{c}}$, $\hat{\bm{\tau}}_i$, and $\hat{\chi}_i$ to be estimates of the state variables, and $\tilde{\mathbf{c}} = \mathbf{c} - \hat{\mathbf{c}}$, $\tilde{\bm{\tau}}_i = \bm{\tau}_i - \hat{\bm{\tau}}_i$, and $\tilde{\chi}_i = \chi_i - \hat{\chi}_i$ are the state estimation errors.
An observer for the system is described as
\begin{equation}
    \begin{split}
        \dot{\hat{\mathbf{c}}} & = \mathbf{Q}\bm{\omega}_c + k_{\mathbf{c}}\tilde{\mathbf{c}} \\
        \dot{\hat{\bm{\tau}}}_i & = \mathbf{T}_i(\bm{\tau}_i,\mathbf{c})^T\bm{\nu}_c \hat{\chi}_i + k_{\bm{\tau}_i}\tilde{\bm{\tau}}_i \\
        \dot{\hat{\chi}}_i & = \mathbf{X}_i(\bm{\tau}_i,\mathbf{c})^T\bm{\nu}_c \hat{\chi}_i^2 + k_{\chi_i}(\mathbf{T}_i(\bm{\tau}_i,\mathbf{c})^T\bm{\nu}_c)^T\tilde{\bm{\tau}}_i,
    \end{split}
    \label{eq:mwobsdyn}
\end{equation}
where $k_{\mathbf{c}}$, $k_{\bm{\tau}_i}$, and $k_{\chi_i}$ are positive gains.
The error dynamics are then given as
\begin{equation}
    \begin{split}
        \dot{\tilde{\mathbf{c}}} & = - k_{\mathbf{c}}\tilde{\mathbf{c}} \\
        \dot{\tilde{\bm{\tau}}}_i & = \mathbf{T}_i(\bm{\tau}_i,\mathbf{c})^T\bm{\nu}_c \tilde{\chi}_i - k_{\bm{\tau}_i}\tilde{\bm{\tau}}_i \\
        \dot{\tilde{\chi}}_i & = \mathbf{X}_i(\bm{\tau}_i,\mathbf{c})^T\bm{\nu}_c (\chi_i + \hat{\chi}_i)\tilde{\chi}_i - k_{\chi_i}(\mathbf{T}_i(\bm{\tau}_i,\mathbf{c})^T\bm{\nu}_c)^T\tilde{\bm{\tau}}_i.
    \end{split}
    \label{eq:mwerrordyn}
\end{equation}

We now study the stability of the error dynamics in \cref{eq:mwerrordyn}, namely of the equilibrium point $[\tilde{\mathbf{c}},\tilde{\bm{\tau}}_i,\tilde{\chi}_i]^T = \mathbf{0}$.
Let us consider the following assumption.
\begin{assumption}
    The depth of a line must be positive at all times, \emph{i.e.}, $l > 0$.
    \label{assump:depth}
\end{assumption}
From the definition of the line depth, $l \geq 0$. We restrict this further since if $l = 0$, we reach a geometrical singularity. Specifically, the projection of a line is a single point in the image.
We can now state the following theorem.
\begin{theorem}
    If \cref{assump:depth} holds along with:
    \begin{enumerate}
        \item the cosine of the angle between the linear velocity and the line closest point is negative, if the depth estimate is positive, and zero otherwise, $\begin{array}{ll}
              \mathbf{X}_i(\bm{\tau}_i,\mathbf{c})^T\bm{\nu}_c \leq 0 & \text{if} \; \hat{\chi}_i > 0\\
              \mathbf{X}_i(\bm{\tau}_i,\mathbf{c})^T\bm{\nu}_c = 0 & \text{otherwise}
            \end{array}$,
        \item the persistence of excitation \cite[Lemma A.3]{marino2010} is verified, $\bm{\nu}_c^T\mathbf{T}_i(\bm{\tau}_i,\mathbf{c})\mathbf{T}_i(\bm{\tau}_i,\mathbf{c})^T\bm{\nu}_c > 0$,
    \end{enumerate}
    with $i = 1,...,N$, then, $[\tilde{\mathbf{c}},\tilde{\bm{\tau}}_i,\tilde{\chi}_i]^T = \mathbf{0}$ is asymptotically stable.
    \label{theo:lineStable}
\end{theorem}
\begin{proof}
    Let us consider the Lyapunov candidate function given as
    \begin{equation}
        V(\tilde{\mathbf{c}},\tilde{\bm{\tau}}_i,\tilde{\chi}_i) = \frac{1}{2}\sum_{i = 1}^{N} \left( \tilde{\bm{\tau}}_i^T\tilde{\bm{\tau}}_i + \frac{1}{k_{\chi_i}} \tilde{\chi}_i^2 \right) + \frac{1}{2} \tilde{\mathbf{c}}^T\tilde{\mathbf{c}}.
        \label{eq:lyapLine}
    \end{equation}
    Taking the time derivative and replacing \cref{eq:mwerrordyn} yields
    \begin{multline}
        \dot{V}(\tilde{\mathbf{c}},\tilde{\bm{\tau}}_i,\tilde{\chi}_i) = -k_{\mathbf{c}}\|\tilde{\mathbf{c}}\|^2 + \sum_{i = 1}^N \left( -k_{\bm{\tau}_i}\|\tilde{\bm{\tau}}_i\|^2 + \right. \\ \left. \frac{1}{k_{\chi_i}}\mathbf{X}_i(\bm{\tau}_i,\mathbf{c})^T\bm{\nu}_c (\chi_i + \hat{\chi}_i)\tilde{\chi}_i^2 \right).
        \label{eq:lyapunovDerivative}
    \end{multline}
    From \cref{assump:depth}, and the conditions in \cref{theo:lineStable}, we conclude that all terms are non-positive, and thus $\dot{V} \leq 0$, and $[\tilde{\mathbf{c}},\tilde{\bm{\tau}}_i,\tilde{\chi}_i]^T = \mathbf{0}$ is stable. Furthermore, we can conclude that, the Lyapunov candidate function in \cref{eq:lyapLine} is decrescent -- as a function of time -- and thus the signals $[\tilde{\mathbf{c}},\tilde{\bm{\tau}}_i,\tilde{\chi}_i]^T$ are bounded.
    
    The critical scenario to show asymptotically stability is $\mathbf{X}_i(\bm{\tau}_i,\mathbf{c})^T\bm{\nu}_c = 0$, with $i = 1,...,N$, which can happen if the linear velocity has the same direction as the line.
    For the sake of simplicity, we address the case when all $\mathbf{X}_i(\bm{\tau}_i,\mathbf{c})^T\bm{\nu}_c = 0$, the case when only a subset is null, can be shown similarly.
    
    Let us consider the second time derivative of $V$ -- for the scenario described above -- which is given as
    \begin{multline}
        \ddot{V} (\tilde{\mathbf{c}},\tilde{\bm{\tau}}_i,\tilde{\chi}_i) = 2 k_{\mathbf{c}}^2\|\tilde{\mathbf{c}}\|^2 + \sum_{i = 1}^N \left( 2 k_{\bm{\tau}_i}^2\|\tilde{\bm{\tau}}_i\|^2 - \right. \\ \left.
        k_{\bm{\tau}_i}\bm{\tau}_i^T\mathbf{T}_i(\bm{\tau}_i,\mathbf{c})^T\bm{\nu}_c \tilde{\chi}_i  \right),
    \end{multline}
    which is composed of bounded signals, and thus is bounded. So, $\dot{V}$ is uniformly continuous and by Barbalat's Lemma \cite[Lemma8.2]{khalil2015}, we conclude that $\dot{V} \rightarrow 0$ as $t \rightarrow \infty$, and thus $\tilde{\mathbf{c}} \rightarrow 0$, and $\tilde{\bm{\tau}}_i \rightarrow 0$.
    To assess asymptotic stability of $\tilde{\chi}_i$, let us compute the second time derivative of $\tilde{\bm{\tau}}_i$. Notice that, the analysis of the signal $\tilde{\mathbf{c}}$ is not considered since it is not influenced by $\tilde{\chi}_i$. 
    The second derivative is
    \begin{multline}
        \ddot{\tilde{\bm{\tau}}}_i = - k_{\bm{\tau}_i}^2\tilde{\bm{\tau}}_i +  \dot{\mathbf{T}}_i(\bm{\tau}_i,\mathbf{c})^T\bm{\nu}_c \tilde{\chi}_i +\\ \mathbf{T}_i(\bm{\tau}_i,\mathbf{c})^T\dot{\bm{\nu}}_c \tilde{\chi}_i +  \mathbf{T}_i(\bm{\tau}_i,\mathbf{c})^T\bm{\nu}_c \dot{\tilde{\chi}}_i,
    \end{multline}
    which is a function of bounded signals. Thus by applying again Barbalat's Lemma, we conclude that $\dot{\tilde{\bm{\tau}}}_i \rightarrow 0$ as $t \rightarrow \infty$. Taking the limit, we obtain
    \begin{equation}
        \lim_{t\to\infty} \dot{\tilde{\bm{\tau}}}_i = \lim_{t\to\infty} \mathbf{T}_i(\bm{\tau}_i,\mathbf{c})^T\bm{\nu}_c \lim_{t\to\infty} \tilde{\chi}_i = 0.
    \end{equation}
    So either $\tilde{\chi}_i \rightarrow 0$ or $\mathbf{T}_i(\bm{\tau}_i,\mathbf{c})^T\bm{\nu}_c \rightarrow 0$.
    
    If the signal $\mathbf{T}_i(\bm{\tau}_i,\mathbf{c})^T\bm{\nu}_c$ is persistently exciting, then the lines depth estimation error is asymptotically stable.
    The persistence of excitation condition \cite[Lemma A.3]{marino2010} is verified if $\bm{\nu}_c^T\mathbf{T}_i(\bm{\tau}_i,\mathbf{c})\mathbf{T}_i(\bm{\tau}_i,\mathbf{c})^T\bm{\nu}_c > 0$ with $i = 1,...,N$, which by assumption holds.
    Thus the origin of the dynamical system in \cref{eq:mwerrordyn} is asymptotically stable.
\end{proof}
\section{Linear Velocity and Line Estimation}
\label{sec:velPointDepthEstimation}

This section presents a scheme to estimate both the camera linear velocity and the 3D lines in Manhattan World.
The approach consists of first estimating the linear velocity and the depth of a planar target using the observer in \cref{eq:velPlaneDepthObs}.
Then, the linear velocity estimate is used as input to the observer in \cref{sec:mwLinesObs}.

Let us make the following assumption
\begin{assumption}
    A planar target is in front of the camera, \emph{i.e.}, $\rho > 0$.
    \label{ass:positiveDepth}
\end{assumption}
This assumption holds in practice, since we are considering perspective cameras, and thus points/planes behind the optical center cannot be observed.

Let us consider the observer in \cref{sec:mwLinesObs} with inputs $(\hat{\bm{\nu}}_c,\bm{\omega}_c)$, where $\hat{\bm{\nu}}_c = \frac{\hat{\mathbf{s}}}{\hat{\psi}}$.
The error dynamics in \cref{eq:mwerrordyn}, are now given as
\begin{equation}
    \begin{split}
        \dot{\tilde{\mathbf{c}}} & = - k_{\mathbf{c}}\tilde{\mathbf{c}} \\
        \dot{\tilde{\bm{\tau}}}_i & = \mathbf{T}_i(\bm{\tau}_i,\mathbf{c})^T (\bm{\nu}_c\chi - \hat{\bm{\nu}}_c\hat{\chi}_i) - k_{\bm{\tau}_i}\tilde{\bm{\tau}}_i \\
        \dot{\tilde{\chi}}_i & = \mathbf{X}_i(\bm{\tau}_i,\mathbf{c})^T(\bm{\nu}_c\chi_i^2-\hat{\bm{\nu}}_c\hat{\chi}_i^2) - \\ & \hspace{70pt} k_{\chi_i}(\mathbf{T}_i(\bm{\tau}_i,\mathbf{c})^T(\hat{\bm{\nu}}_c))^T\tilde{\bm{\tau}}_i.
    \end{split}
    \label{eq:mwerrordynEst}
\end{equation}

We can now describe the main result of this section.
\begin{theorem}
    If \cref{ass:positiveDepth} and the Persistence of excitation\cite[Lemma A.3]{marino2010} hold for the observer proposed in \cite{grabe2015} and the one in \cref{eq:mwobsdyn}) along with the conditions in \cref{theo:lineStable}, then, $[\tilde{\mathbf{c}},\tilde{\bm{\tau}}_i,\tilde{\chi}_i]^T = \mathbf{0}$ is an asymptotically stable equilibrium point of the system in \cref{eq:mwerrordyn} with inputs $(\hat{\bm{\nu}}_c,\bm{\omega}_c)$.
    \label{theo:velPointDepthEst}
\end{theorem}

\begin{proof}
    Let us consider the Lyapunov candidate function in \cref{eq:lyapLine}. Taking the time derivative and using \cref{eq:mwerrordynEst}, we get 
    \begin{multline}
        \dot{V}(\tilde{\mathbf{c}},\tilde{\bm{\tau}}_i,\tilde{\chi}_i) = -k_{\mathbf{c}}\|\tilde{\mathbf{c}}\|^2 + \sum_{i = 1}^N \Big( -k_{\bm{\tau}_i}\|\tilde{\bm{\tau}}_i\|^2 + \\
        \left. \tilde{\bm{\tau}}_i^T(\mathbf{T}_i(\bm{\tau}_i,\mathbf{c})^T (\bm{\nu}_c\chi - \hat{\bm{\nu}}_c\hat{\chi}_i)) - \tilde{\chi}_i (\mathbf{T}_i(\bm{\tau}_i,\mathbf{c})^T\hat{\bm{\nu}}_c)^T\tilde{\bm{\tau}}_i + \right. \\
        \frac{1}{k_{\chi_i}}\tilde{\chi}_i \mathbf{X}_i(\bm{\tau}_i,\mathbf{c})^T(\bm{\nu}_c\chi_i^2-\hat{\bm{\nu}}_c\hat{\chi}_i^2)  \Big).
        \label{eq:lyapunovDerivativeCascade}
    \end{multline}
    Exploiting the fact that $\bm{\nu}_c \chi_i -\hat{\bm{\nu}}_c \hat{\chi}_i = \tilde{\bm{\nu}}_c \chi_i + \hat{\bm{\nu}}_c \tilde{\chi}_i$, with $\tilde{\bm{\nu}}_c = \bm{\nu}_c - \hat{\bm{\nu}}_c$, and replacing in \cref{eq:lyapunovDerivativeCascade}, yields
    \begin{multline}
        \dot{V}(\tilde{\mathbf{c}},\tilde{\bm{\tau}}_i,\tilde{\chi}_i) = -k_{\mathbf{c}}\|\tilde{\mathbf{c}}\|^2 + \sum_{i = 1}^N \big( -k_{\bm{\tau}_i}\|\tilde{\bm{\tau}}_i\|^2 + \\
        \tilde{\bm{\tau}}_i^T(\mathbf{T}_i(\bm{\tau}_i,\mathbf{c})^T \tilde{\bm{\nu}}_c\chi_i + \frac{1}{k_{\chi_i}}\tilde{\chi}_i \mathbf{X}_i(\bm{\tau}_i,\mathbf{c})^T(\bm{\nu}_c\chi_i^2-\hat{\bm{\nu}}_c\hat{\chi}_i^2)  \big).
        \label{eq:lyapunovDerivativeCascade2}
    \end{multline}
    Finally, by applying $\bm{\nu}_c \chi_i^2 - \hat{\bm{\nu}}_c \hat{\chi}_i^2 = \bm{\nu}_c(\chi_i^2 - \hat{\chi}_i^2) + \tilde{\bm{\nu}}_c\hat{\chi}_i^2$ in \cref{eq:lyapunovDerivativeCascade2}, we obtain
    \begin{multline}
        \dot{V}(\tilde{\mathbf{c}},\tilde{\bm{\tau}}_i,\tilde{\chi}_i) = -k_{\mathbf{c}}\|\tilde{\mathbf{c}}\|^2 + \sum_{i = 1}^N \Big( -k_{\bm{\tau}_i}\|\tilde{\bm{\tau}}_i\|^2 + \\
        \left. \tilde{\bm{\tau}}_i^T(\mathbf{T}_i(\bm{\tau}_i,\mathbf{c})^T \tilde{\bm{\nu}}_c\chi_i + \frac{1}{k_{\chi_i}} \mathbf{X}_i(\bm{\tau}_i,\mathbf{c})^T \bm{\nu}_c (\chi_i + \hat{\chi}_i)\tilde{\chi}_i^2 + \right. \\
        \frac{1}{k_{\chi_i}} \mathbf{X}_i(\bm{\tau}_i,\mathbf{c})^T\tilde{\bm{\nu}}_c \hat{\chi}_i^2\tilde{\chi}_i \Big).
        \label{eq:lyapunovDerivativeCascade3}
    \end{multline}
    From \cite{grabe2015}, $\tilde{\bm{\nu}}_c$ converges exponentially to zero as time goes to infinity.
    The remaining terms in \cref{eq:lyapunovDerivativeCascade3} give us the same $\dot{V}$ as in \cref{eq:lyapunovDerivative}, where it is proven that the origin is asymptotically stable.
\end{proof}
\section{Results}
\label{sec:results}

This section presents the simulation results. 
We start by evaluating the depth observer for lines in Manhattan World, presented in \cref{sec:mwLinesObs} -- henceforward referred to as \emph{MWLEst} --, and compare to the observer in \cite{mateus2019} -- henceforward referred to as \emph{$N$-line Sphere}.
However, that observer considers a single line. Thus we expand its state space to account for multiple lines.
We assess the convergence time and failure rate of both observers.
Then, the same observers are evaluated with noisy data.
Finally, we show the convergence of the cascade in a simulated environment.

\subsection{Depth Observer for Lines in Manhattan World}
\label{sec:resultsObs}

The evaluation consisted of two tests conducted with MATLAB. 
We start by validating the \emph{MWLEst} method with noiseless data against the \emph{$N$-line Sphere} method in terms of percentage of success, convergence speed, and traveled distance.
Then, we run experiments with increasing noise levels to evaluate the robustness of the methods to measurement noise.
Notice that the evaluation metrics for the noise and noiseless scenarios are different. The estimation error is not ensured to converge to zero with noise, only to a bounded region about zero.
Furthermore, computing direction and depth errors in noiseless data will yield zero for both methods.
For all the experiments the gain of the \emph{$N$-line Sphere} observer, and the gains $k_{\chi_i}$ were set to $100$. The gains $k_{\mathbf{c}_i}$ and $k_{\bm{\tau}_i}$ were set to $2\sqrt{k_{\chi_i}}$.

\vspace{5pt}\noindent{ Data Generation:}
A perspective camera was used as the imaging device, whose intrinsic parameters (for more details see \cite{hartley2003}) are defined as $\mathbf{I}_3 \in \mathbb{R}^{3\times3}$. 
All six degrees of freedom (DoF) are assumed to be controllable. 
The Manhattan World frame is generated by randomly selecting an orthogonal matrix $\mathbf{R}_{\mathcal{C}}^{\mathcal{W}}$.
A point for each line is also randomly drawn in a $25$ side cube in front of the camera.
The points are then assigned to a principal direction, and the moments and depth are computed.
The Cayley parameters are computed from $\mathbf{R}_{\mathcal{C}}^{\mathcal{W}}$, as explained in \cref{sec:mwLines}.
Finally, the vectors $\bm{\tau}_i$ are computed.
The initial state of the measured quantities is set to its actual value. The unknowns are selected randomly.
Since the unknown variables of the two methods differ but have the same scale (the inverse of the depth), the unknown quantities in the \emph{$N$-line Sphere} method are set with unit norm and then scaled with the initial scale used by the \emph{MWLEst} method.

\begin{table}[t]
    \centering
    \resizebox{1\linewidth}{!}{\setlength{\tabcolsep}{5.0pt}\begin{tabular}{@{}l@{\hskip 5pt}ccc@{}}
        \toprule
        \thead[l]{Method} & \thead{Percentage \\ of Success} & \thead{Convergence \\ Time (s)} & \thead{Distance \\ (m)} \\ \midrule
        MWLEst & 88.6 & 4.8 & 2.46 \\ 
        $N$-line Sphere & 28.3 & 6.08 & 2.05 \\
        \toprule
        \end{tabular}
    }
    \caption{\it Performance evaluation of the \emph{MWLEst} method with respect to the \emph{$N$-line Sphere}. For 1000 randomly generated trials, we compare both methods in terms of the success rate (\% of Success), the median of the convergence time (Conv. Time (s)), and the median of the distance the agent took (Median Distance (m)).}
    \label{tab:evaluation_noiseless}
\end{table}

\vspace{5pt}\noindent{\bf Noiseless Data:}
The first test consists of running both methods (\emph{MWLEst} and the \emph{$N$-line Sphere}) for $1000$ distinct trials in a noise-free scenario.
This test evaluates the percentage of success of both observers\footnote{We consider that the method succeeded if it did not diverge.}, the median convergence time\footnote{We assume that the method converged when the error is less than 1\% of the initial error.}, and the median distance traveled by the camera. 
The results are shown in \cref{tab:evaluation_noiseless}. As we can see from these results, \emph{MWLEst} beats the \emph{$N$-line Sphere} by some margin in both the percentage of success and convergence time. We see that the \emph{MWLEst} travels for a longer distance than the \emph{$N$-line Sphere} but converges faster.

\begin{figure}
    \centering
    \subfloat[Median direction error of $1000$ runs with different noise levels.]{
        \includegraphics[width=0.45\textwidth]{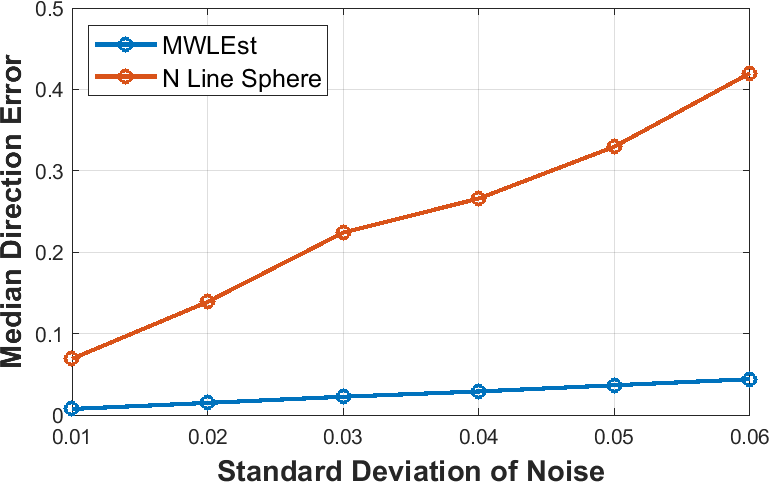}
        \label{fig:mwdirectionError}
    } \hfill
    \subfloat[Median depth error of $1000$ runs with different noise levels.]{
        \includegraphics[width=0.45\textwidth]{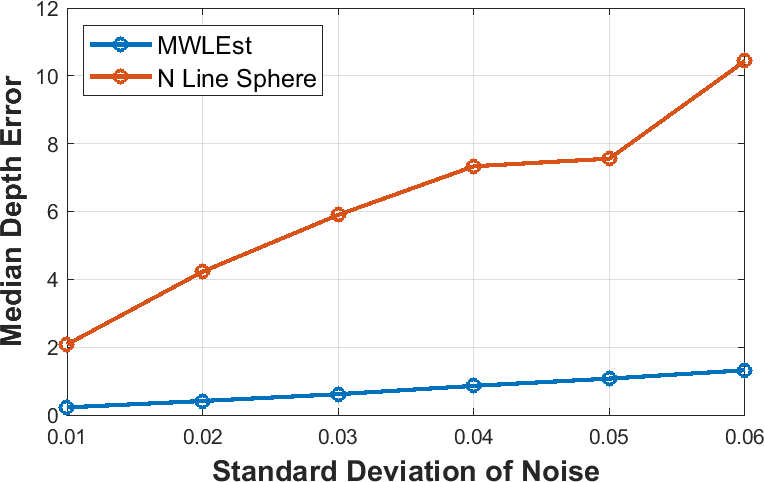}
        \label{fig:mwdepthError}
    }
    \caption{\it Median estimation errors of both methods for $1000$ runs with different standard deviations of the error. On the left, the estimation error of the direction vectors. Keep in mind that the \emph{MWLEst} method can recover the directions from the image, while the \emph{$N$-line Sphere} estimates them. On the right, the estimation error of the line depths.}
    \label{fig:simmwNoiseResults}
\end{figure}

\vspace{5pt}\noindent{\bf Noisy Data:}
Noise was added to the measurements, i.e., the Manhattan frame and the projection of the moment vector to that frame. 
A rotation matrix was applied to the Manhattan frame, with Euler angles sampled from a uniform distribution with zero mean and increasing standard deviation. 
Noise was added to the moments with a rotation since they are unit vectors
The noisy moments are then projected to the noisy frame, yielding the parameters $\tau_{ij}$.
Six different noise levels were considered, and $1000$ randomly generated runs were executed for each level.
The moment vector can be measured in both methods; thus, its error is not presented.
However, we stress that the error of the moment vector is in the same order of magnitude as the added error.
The direction and depth errors are defined as
\begin{equation}
    \epsilon_{\mathbf{d}} = \arccos(\hat{\mathbf{d}}_i^T\mathbf{d}_i), \ \text{and} \ \
    \epsilon_l =  \| \hat{l}_i - l_i \|.
    \label{eq:errors}
\end{equation}
The median direction and depth errors for each noise level considered are presented in \cref{fig:simmwNoiseResults}\subref{fig:mwdirectionError} and \cref{fig:simmwNoiseResults}\subref{fig:mwdepthError} respectively.
We can see that the \emph{MWLEst} outperforms the \emph{$N$-line Sphere} method by a large margin.

\subsection{Observer Cascade}
\label{sec:resultsCascade}

\begin{figure}[t]
    \centering
    \subfloat[State estimation error of the observer in \cref{eq:velPlaneDepthObs}.]{
        \includegraphics[width=0.4\textwidth]{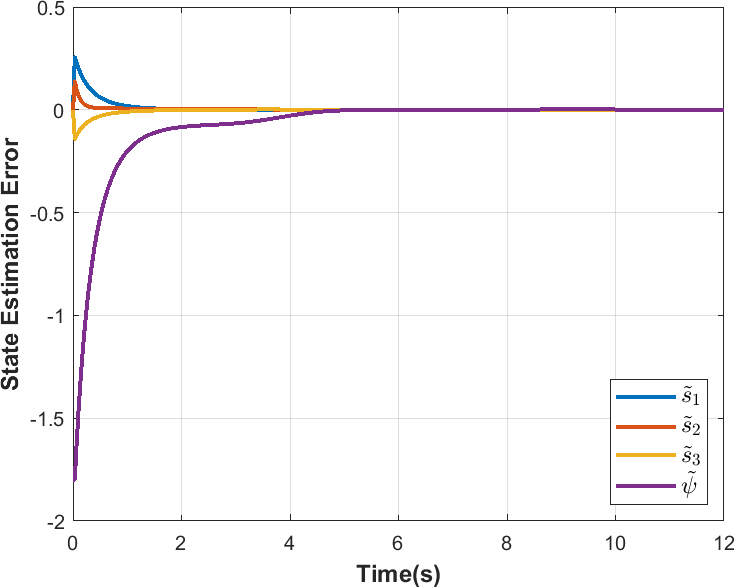}
        \label{fig:velDepthPlaneSim}
    } \,    
    \subfloat[State estimation error of the MW lines depth observer in \cref{eq:mwobsdyn} with the estimated velocity.]{
        \includegraphics[width=0.4\textwidth]{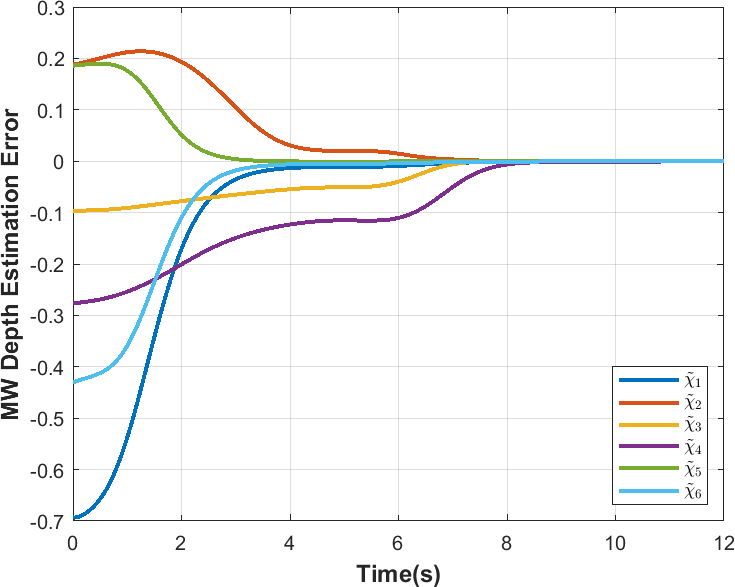}
        \label{fig:velDepthMWSim}
    }
    \caption{\it State estimation errors of the observer in \cref{eq:velPlaneDepthObs}, and \cref{eq:mwobsdyn}. The plane depth observer is used to estimate the camera linear velocity, which is inputted into the MWLEst.}
    \label{fig:velDepthSimResults}
\end{figure}

This section presents the results of the observer cascade.
For that purpose, a perspective camera was considered, with its intrinsic parameters given by an identity matrix.
The angular velocity and linear acceleration were defined with sinusoidal signals to ease integration and differentiation. Each component has a phase shift with respect to the others, so that every signal has a different value.
The signals were select such that the maximum acceleration norm was $\text{max} \| \mathbf{a}_{\mathcal{I}} \| \simeq 2$, and the maximum norm of the angular velocity was $\text{max} \| \bm{\omega}_{\mathcal{I}} \| \simeq 0.5$.

The plane was chosen to be parallel with the ground plane, with the camera facing down, i.e., the normal of the plane at the start of the experiment is parallel to the camera's optical axis.
The plane depth $\rho$ was selected randomly with a uniform distribution centered at five units.
Lines were generated similarly to \cref{sec:resultsObs}.

The simulation consisted of using the two observers at the same time for $12$ seconds.
The plane depth observer is used to retrieve the scale of the linear velocity.
That estimate is then fed to the MW line depth observer for six lines.
The state estimation error of plane depth observer is presented in \cref{fig:velDepthSimResults}\subref{fig:velDepthPlaneSim}. The gains were set as $k_{\mathbf{s}} = 2$, and $k_{\rho} = 20$.
The observer converged in approximately $4.33$ seconds.
The state estimation error of the six lines using the \emph{MWLEst} is presented in \cref{fig:velDepthSimResults}\subref{fig:velDepthMWSim}.
The gains were set as $k_{\mathbf{c}} = 20$, $k_{\bm{\tau}_i} = 20$, and $k_{\chi_i} = 200$.
The observer converged in $7.9$ seconds.
\section{Conclusions}
\label{sec:conclusions}

We proposed a method to estimate multiple lines.
By building the model using the Manhattan World Assumption, we show that the number of state variables is reduced from $4N$ to $3N+3$.
Since the orthogonal direction, yielding the Manhattan frame, can be retrieved from a single image, only the lines' depth are unknown. 
In the previous methods, both the direction and the depth need to be estimated.
A model for lines in a Manhattan world is proposed. The dynamics of this model are derived by taking advantage of the Cayley transform to have a minimal representation of the Manhattan frame.
An observer is presented to estimate the Manhattan world, which is asymptotically stable.
The approach is compared in simulation with an extension of a state-of-the-art observer and is shown to diverge less often, converge faster, and be more robust to measurement noise.

Furthermore, the previous methods assume that the camera velocities are known.
In this work, we relaxed that assumption by using an IMU coupled with the camera.
This allows us to obtain the linear velocity by exploiting a state-of-the-art observer for plane depth estimation.
We then proposed an observer cascade where the estimate from the above observer is used as the input to the MW line depth observer.
Stability analysis of the cascade is presented, showing that the state estimation error is asymptotically stable, with the linear velocity given by the first observer.

Future work includes applying the observer cascade to real data and developing Active Vision strategies that allow us to define control inputs that optimize the convergence of both observers.

\clearpage

\bibliographystyle{IEEEtran}
\bibliography{citations}

\begin{thebibliography}{10}
\providecommand{\url}[1]{#1}
\csname url@rmstyle\endcsname
\providecommand{\newblock}{\relax}
\providecommand{\bibinfo}[2]{#2}
\providecommand\BIBentrySTDinterwordspacing{\spaceskip=0pt\relax}
\providecommand\BIBentryALTinterwordstretchfactor{4}
\providecommand\BIBentryALTinterwordspacing{\spaceskip=\fontdimen2\font plus
\BIBentryALTinterwordstretchfactor\fontdimen3\font minus
  \fontdimen4\font\relax}
\providecommand\BIBforeignlanguage[2]{{%
\expandafter\ifx\csname l@#1\endcsname\relax
\typeout{** WARNING: IEEEtran.bst: No hyphenation pattern has been}%
\typeout{** loaded for the language `#1'. Using the pattern for}%
\typeout{** the default language instead.}%
\else
\language=\csname l@#1\endcsname
\fi
#2}}

\bibitem{koenderink1991}
J.~Koenderink and A.~Van~Doorn, ``Affine structure from motion,'' \emph{Journal
  of the Optical Society of America A}, vol.~8, no.~2, pp. 377--385, 1991.

\bibitem{bartoli2005}
A.~Bartoli and P.~Sturm, ``Structure-from-motion using lines: Representation,
  triangulation, and bundle adjustment,'' \emph{Computer Vision and Image
  Understanding (CVIU)}, vol. 100, no.~3, pp. 416--441, 2005.

\bibitem{kalman1960}
R.~E. Kalman, ``A new approach to linear filtering and prediction problems,''
  \emph{Transactions of the ASME--Journal of Basic Engineering}, vol.~82, no.
  Series D, pp. 35--45, 1960.

\bibitem{civera2008}
J.~Civera, A.~J. Davison, and J.~M. Montiel, ``Inverse depth parametrization
  for monocular slam,'' \emph{IEEE Trans. Robotics (T-RO)}, vol.~24, no.~5, pp.
  932--945, 2008.

\bibitem{civera2010}
J.~Civera, O.~G. Grasa, A.~J. Davison, and J.~M. Montiel, ``1-point ransac for
  extended kalman filtering: Application to real-time structure from motion and
  visual odometry,'' \emph{Journal of Field Robotics}, vol.~27, no.~5, pp.
  609--631, 2010.

\bibitem{omari2013}
S.~Omari and G.~Ducard, ``Metric visual-inertial navigation system using single
  optical flow feature,'' in \emph{IEEE European Control Conference (ECC)},
  2013, pp. 1310--1316.

\bibitem{smith2006}
P.~Smith, I.~Reid, and A.~Davison, ``Real-time monocular slam with straight
  lines,'' in \emph{British Machine Vision Conference (BMVC)}, 2006, pp.
  17--26.

\bibitem{zhang2011}
G.~Zhang and I.~H. Suh, ``Building a partial 3d line-based map using a
  monocular slam,'' in \emph{IEEE Int'l Conf. Robotics and Automation (ICRA)},
  2011, pp. 1497--1502.

\bibitem{chaumette2006}
F.~{C}haumette and S.~{H}utchinson, ``Visual {S}ervo {C}ontrol. {I}. {B}asic
  {A}pproaches,'' \emph{IEEE Robotics Automation Magazine (RA-M)}, vol.~13,
  no.~4, pp. 82--90, 2006.

\bibitem{dixon2003}
W.~E. Dixon, Y.~Fang, D.~M. Dawson, and T.~J. Flynn, ``Range identification for
  perspective vision systems,'' \emph{IEEE Trans. Automatic Control (T-AC)},
  vol.~48, no.~12, pp. 2232--2238, 2003.

\bibitem{deluca2008}
A.~De~Luca, G.~{O}riolo, and P.~Robuffo~Giordano, ``Feature {D}epth
  {O}bservation for {I}mage-{B}ased {V}isual {S}ervoing: {T}heory and
  {E}xperiments,'' \emph{The International Journal of Robotics Research
  (IJRR)}, vol.~27, no.~10, pp. 1093--1116, 2008.

\bibitem{morbidi2010}
F.~Morbidi, G.~L. Mariottini, and D.~Prattichizzo, ``Observer design via
  immersion and invariance for vision-based leader--follower formation
  control,'' \emph{Automatica}, vol.~46, no.~1, pp. 148--154, 2010.

\bibitem{sassano2010}
M.~Sassano, D.~Carnevale, and A.~Astolfi, ``Observer design for range and
  orientation identification,'' \emph{Automatica}, vol.~46, no.~8, pp.
  1369--1375, 2010.

\bibitem{dani2012}
A.~P. Dani, N.~R. Fischer, and W.~E. Dixon, ``Single camera structure and
  motion,'' \emph{IEEE Trans. Automatic Control (T-AC)}, vol.~57, no.~1, pp.
  238--243, 2012.

\bibitem{rodrigues2019}
R.~T. Rodrigues, P.~Miraldo, D.~V. Dimarogonas, and A.~P. Aguiar, ``A framework
  for depth estimation and relative localization of ground robots using
  computer vision,'' in \emph{IEEE/RSJ Int'l Conf. Intelligent Robots and
  Systems (IROS)}, 2019, pp. 3719--3724.

\bibitem{mateus2021}
A.~Mateus, O.~Tahri, A.~P. Aguiar, P.~U. Lima, and P.~Miraldo, ``On incremental
  structure from motion using lines,'' \emph{IEEE Trans. Robotics (T-RO)},
  2021.

\bibitem{spica2013}
R.~Spica and P.~Robuffo~Giordano, ``A {F}ramework for {A}ctive {E}stimation:
  {A}pplication to {S}tructure from {M}otion,'' in \emph{IEEE Conf. Decision
  and Control (CDC)}, 2013, pp. 7647--7653.

\bibitem{spica2014}
R.~Spica, P.~Robuffo~Giordano, and F.~Chaumette, ``{A}ctive {S}tructure from
  {M}otion: {A}pplication to {P}oint, {S}phere, and {C}ylinder,'' \emph{IEEE
  Trans. Robotics (T-RO)}, vol.~30, no.~6, pp. 1499--1513, 2014.

\bibitem{spica2015}
R.~Spica, P.~R. Giordano, and F.~Chaumette, ``Plane estimation by active vision
  from point features and image moments,'' in \emph{IEEE Int'l Conf. Robotics
  and Automation (ICRA)}, 2015, pp. 6003--6010.

\bibitem{spica2015b}
P.~R. Giordano, R.~Spica, and F.~Chaumette, ``Learning the shape of image
  moments for optimal 3d structure estimation,'' in \emph{IEEE Int'l Conf.
  Robotics and Automation (ICRA)}, 2015, pp. 5990--5996.

\bibitem{tahri2015}
O.~Tahri, P.~R. Giordano, and Y.~Mezouar, ``Rotation free active vision,'' in
  \emph{IEEE/RSJ Int'l Conf. Intelligent Robots and Systems (IROS)}, 2015, pp.
  3086--3091.

\bibitem{tahri2017}
O.~Tahri, D.~Boutat, and Y.~Mezouar, ``Brunovsky's linear form of incremental
  structure from motion,'' \emph{IEEE Trans. Robotics (T-RO)}, vol.~33, no.~6,
  pp. 1491--1499, 2017.

\bibitem{mateus2018}
A.~Mateus, O.~Tahri, and P.~Miraldo, ``Active {S}tructure-from-{M}otion for
  {3D} {S}traight {L}ines,'' in \emph{IEEE/RSJ Int'l Conf. Intelligent Robots
  and Systems (IROS)}, 2018, pp. 5819--5825.

\bibitem{mateus2019}
------, ``Active {E}stimation of {3D} {L}ines in {S}pherical {C}oordinates,''
  in \emph{American Control Conf. (ACC)}, 2019, pp. 3950--3955.

\bibitem{spica2017}
R.~Spica, P.~Robuffo~Giordano, and F.~Chaumette, ``Coupling active depth
  estimation and visual servoing via a large projection operator,'' \emph{The
  International Journal of Robotics Research (IJRR)}, vol.~36, no.~11, pp.
  1177--1194, 2017.

\bibitem{rodrigues2020}
R.~T. Rodrigues, P.~Miraldo, D.~V. Dimarogonas, and A.~P. Aguiar, ``Active
  depth estimation: Stability analysis and its applications,'' in \emph{IEEE
  Int'l Conf. Robotics and Automation (ICRA)}, 2020, pp. 2002--2008.

\bibitem{benyoucef2019}
R.~Benyoucef, L.~Nehaoua, H.~Hadj-Abdelkader, and H.~Arioui, ``Linear camera
  velocities and point feature depth estimation using unknown input observer,''
  in \emph{11th IEEE/RSJ Int'l Conf. on Intelligent Robots and Systems (IROS)
  Workshop on Planning, Perception, Navigation for Intelligent part of
  Vehicle}, 2019.

\bibitem{benyoucef2021}
R.~Benyoucef, H.~Hadj-Abdelkader, L.~Nehaoua, and H.~Arioui, ``Towards
  kinematics from motion: Unknown input observer and dynamic extension
  approach,'' \emph{IEEE Control Systems Letters}, vol.~6, pp. 1340--1345,
  2021.

\bibitem{andreff2002}
N.~Andreff, B.~Espiau, and R.~Horaud, ``Visual {S}ervoing from {L}ines,''
  \emph{The International Journal of Robotics Research (IJRR)}, vol.~21, no.~8,
  pp. 679--699, 2002.

\bibitem{ma2012}
Y.~Ma, S.~Soatto, J.~Kosecka, and S.~S. Sastry, \emph{An invitation to 3-d
  vision: from images to geometric models}.\hskip 1em plus 0.5em minus
  0.4em\relax Springer Science \& Business Media, 2012, vol.~26.

\bibitem{grabe2015}
V.~Grabe, H.~H. B{\"u}lthoff, D.~Scaramuzza, and P.~R. Giordano, ``Nonlinear
  ego-motion estimation from optical flow for online control of a quadrotor
  uav,'' \emph{The International Journal of Robotics Research (IJRR)}, vol.~34,
  no.~8, pp. 1114--1135, 2015.

\bibitem{Coughlan1999}
J.~Coughlan and A.~Yuille, ``Manhattan world: compass direction from a single
  image by bayesian inference,'' in \emph{IEEE Int'l Conf. Computer Vision
  (ICCV)}, vol.~2, 1999, pp. 941--947.

\bibitem{Coughlan2000}
J.~M. Coughlan and A.~L. Yuille, ``The manhattan world assumption: Regularities
  in scene statistics which enable bayesian inference,'' in \emph{Advances in
  Neural Information Processing Systems (NIPS)}, 2000, p. 809–815.

\bibitem{lu2017}
X.~Lu, J.~Yaoy, H.~Li, Y.~Liu, and X.~Zhang, ``2-line exhaustive searching for
  real-time vanishing point estimation in manhattan world,'' in \emph{IEEE
  Winter Conference on Applications of Computer Vision (WACV)}, 2017, pp.
  345--353.

\bibitem{Krantz2012}
S.~G. Krantz, \emph{Handbook of complex variables}.\hskip 1em plus 0.5em minus
  0.4em\relax Springer Science \& Business Media, 2012.

\bibitem{marino2010}
R.~Marino, P.~Tomei, and C.~M. Verrelli, \emph{Induction motor control
  design}.\hskip 1em plus 0.5em minus 0.4em\relax Springer Science \& Business
  Media, 2010.

\bibitem{khalil2015}
H.~K. Khalil, \emph{Nonlinear control}.\hskip 1em plus 0.5em minus 0.4em\relax
  Pearson New York, 2015.

\bibitem{hartley2003}
R.~Hartley and A.~Zisserman, \emph{Multiple View Geometry in Computer
  Vision}.\hskip 1em plus 0.5em minus 0.4em\relax Cambridge University Press,
  2003.

\end{thebibliography}

\end{document}